\documentclass[12pt]{article}
\usepackage{abstract,amsmath,amssymb,latexsym}
\usepackage{enumitem,epsf}
\usepackage{fullpage,tikz,float}
\usepackage[numbers]{natbib}
\usepackage[pdftex,colorlinks]{hyperref}
\usepackage{algorithm}
\usepackage{algorithmic}

\newcommand{\eqdef}{\stackrel{\footnotesize\rm def}{=}}

\usepackage{macros}

\newif\ifdraft
\drafttrue

\title{Neural Networks Learning and Memorization with (almost) no Over-Parameterization}

\author{
	\vspace{1cm}
  Amit Daniely\thanks{Hebrew University and Google} 
}

\begin{document}
\maketitle
\setcounter{page}{0}

\thispagestyle{empty}
\maketitle

\begin{abstract}
Many results in recent years established polynomial time learnability of various models via neural networks algorithms
(e.g. \cite{andoni2014learning, daniely2016toward, daniely2017sgd, cao2019generalization, ziwei2019polylogarithmic, zou2019improved, ma2019comparative, du2018gradient, arora2019fine, song2019quadratic, oymak2019towards, ge2019mildly, brutzkus2018sgd}).
However, unless the model is linear separable~\cite{brutzkus2018sgd}, or the activation is a polynomial~\cite{ge2019mildly}, these results require very large networks -- much more than what is needed for the mere existence of a good predictor.

In this paper we prove that SGD on depth two neural networks can memorize samples, learn polynomials with bounded weights, and learn certain kernel spaces, with {\em near optimal} network size, sample complexity, and runtime. In particular, we show that SGD on depth two network with $\tilde{O}\left(\frac{m}{d}\right)$ hidden neurons (and hence $\tilde{O}(m)$ parameters)  can memorize $m$ random labeled points in $\sphere^{d-1}$.

\end{abstract}

\newpage

\tableofcontents

\newpage

\section{Introduction}
Understanding the models (i.e. pairs $(\cd,f^*)$ of input distribution $\cd$ and target function $f^*$)
on which neural networks algorithms guaranteed to learn a good predictor is at the heart of deep learning theory today.
In recent years, there has been an impressive progress in this direction. It is now known that neural networks algorithms can learn, in polynomial time, linear models, certain kernel spaces, polynomials, and memorization models (e.g. \cite{andoni2014learning, daniely2016toward, daniely2017sgd, cao2019generalization, ziwei2019polylogarithmic, zou2019improved, ma2019comparative, du2018gradient, arora2019fine, song2019quadratic, oymak2019towards, ge2019mildly, brutzkus2018sgd}).

Yet, while such models has been shown to be learnable in polynomial time and polynomial sized networks, the required size (i.e., number of parameteres) of the networks is still very large, unless the model is linear separable~\cite{brutzkus2018sgd}, or the activation is a polynomial~\cite{ge2019mildly}. This means that the proofs are valid for networks whose size is significantly larger then the minimal size of the network that implements a good predictor.

We make a progress in this direction, and prove that certain NN algorithms can learn memorization models, polynomials, and kernel spaces, with {\em near optimal} network size, sample complexity, and runtime (i.e. SGD iterations). 
Specifically we assume that the instance space is $\sphere^{d-1}$ and consider depth $2$ networks with $2q$ hidden neurons. Such networks calculate a function of the form
\[
h_{W,\bu}(\x) = \sum_{i=1}^{2q}u_i\sigma\left(\inner{\w_i,\x}\right)  = \inner{\bu,\sigma\left(W\x\right)}
\]
We assume that the network is trained via SGD, starting from random weights that are sampled from the following variant of Xavier initialization~\cite{glorot2010understanding}: $W$ will be initialized to be a duplication $W = \begin{bmatrix}W'\\W'\end{bmatrix}$ of a matrix  $W'$ of standard Gaussians and $\bu$ will be a duplication of the all-$B$ vector in dimension $q$, for some $B>0$,  with its negation. We will use rather large $B$, that will depend on the model that we want to learn. We will prove the following results

\paragraph*{Memorization} In the problem of memorization, we consider SGD training on top of a sample $S = \left\{(\x_1,y_1),\ldots,(\x_m,y_m)\right\}$. The goal is to understand how large the networks should be, and (to somewhat leaser extent) how many SGD steps are needed in order to memorize $1-\epsilon$ fraction of the examples, where an example is considered memorized if $y_ih(\x_i) > 0$ for the output function $h$.
Most results, assumes that the points are random or ``looks like random" in some sense.

In order to memorize even just slightly more that half of the $m$ examples we need a network with at least $m$ parameters (up to poly-log factors). However, unless $m\le d$ (in which case the points are linearly separable), best know results require much more than $m$ parameters, and the current state of the art results~\cite{song2019quadratic, oymak2019towards} require $m^2$ parameters. We show that if the points are sampled uniformly at random from $\sphere^{d-1}\times\{\pm 1\}$, then {\em any fraction} of the examples can be memorized by a network with $\tilde{O}(m)$ parameters, and $\tilde{O}\left(\frac{m}{\epsilon^2}\right)$ SGD iterations. Our result is valid for the hinge loss, and most popular activation functions, including the ReLU.

\paragraph*{Bounded distributions}
Our results for polynomials and kernels will depend on what we call the boundedness of the data distribution. We say that a distribution $\cd$ on $\sphere^{d-1}$ is $R$-bounded if for every $\bu\in\sphere^{d-1}$, $\E_{\x\sim\cd}\inner{\bu,\x}^2 \le \frac{R^2}{d}$.
 To help the reader to calibrate our results, we first note that by Cauchy-Schwartz, 
any distribution $\cd$ is $\sqrt{d}$-bounded, and this bound is tight in the cases that $\cd$ is supported on a single point. Despite that, many distributions of interest are $O(1)$-bounded or even $\left(1+o(1)\right)$-bounded. This includes the uniform distribution on $\sphere^{d-1}$, the uniform distribution on the discrete cube $\left\{\pm \frac{1}{\sqrt{d}}\right\}^d$, the uniform distribution on $\Omega\left(d\right)$ random points, and more (see section \ref{sec:bounded}). For simplicity, we will phrase our results in the introduction for $O(1)$-bounded distribution. We note that if the distribution is $R$-bounded (rather than $O(1)$-bounded), our results suffer a multiplicative factor of $R^2$ in the number of parameters, and remains the same in the runtime (SGD steps).

\paragraph*{Polynomials}
For the sake of clarity, we will describe our result for learning even polynomials, with ReLU networks, and the loss being the logistic loss or the hinge loss. Fix a constant integer $c>0$ and consider the class of even polynomials of degree $\le c$ and coefficient vector norm at most $M$. Namely,
\[
\cp^M_{c} = \left\{ p(\x) = \sum_{|\alpha | \text{ is even and }\le c}a_\alpha\x^{\alpha} : \sum_{|\alpha | \text{ is even and }\le c}a^2_\alpha \le M^2 \right\}
\]
where for $\alpha\in \{0,1,2,\ldots\}^d$ and $\x\in\reals^d$ we denote $\x^\alpha = \prod_{i=1}^dx_i^{\alpha_i}$ and $|\alpha | = \sum_{i=1}^d\alpha_i$. 
Learning the class $\cp^M_d$ requires a networks with at least $\Omega\left(M^2\right)$ parameters (and this remains true even if we restrict to $O(1)$-bounded distributions). We show that for $O(1)$-bounded distributions, SGD learns $\cp^M_{c}$, with error parameter $\epsilon$ (that is, it returns a predictor with error $\le \epsilon$), using a network with $\tilde{O}\left(\frac{M^2}{\epsilon^2} \right)$ parameters and $O\left(\frac{M^2}{\epsilon^2} \right)$ SGD iterations.

\paragraph*{Kernel Spaces}
The connection between networks and kernels has a long history (early work inclues \cite{williams1997infinite, rahimi2009weighted} for instance). In recent years, this connection was utilized to analyze the capability of neural networks algorithm (e.g. \cite{andoni2014learning, daniely2016toward, daniely2017sgd, cao2019generalization, ziwei2019polylogarithmic, zou2019improved, ma2019comparative, du2018gradient, arora2019fine, song2019quadratic, oymak2019towards, ge2019mildly}). In fact, virtually all known non-linear learnable models, including memorization models, polynomials, and kernel spaces utilize this connection.
Our paper is not different, and our result for polynomials is a corollary of a more general result about learning certain kernel spaces, that we describe next. Our result about memorization is not a direct corollary, but is also a refinement of that result.
We consider the kernel $k:\sphere^{d-1}\times\sphere^{d-1} \to \reals$ given by
\begin{equation}\label{eq:ntk_intro}
k(\x,\y) = \inner{\x,\y}\cdot \E_{\w\sim\cn(I,0)}\sigma'\left(\inner{\w,\x},\inner{\w,\y}\right)
\end{equation}
which is a variant of the Neural Tangent Kernel~\cite{jacot2018neural} (see section \ref{sec:ntk}). We show that for $O(1)$-bounded distributions, SGD learns functions with norm $\le M$ in the corresponding kernel space, with error parameter $\epsilon$, using a network with $\tilde{O}\left(\frac{M^2}{\epsilon^2} \right)$ parameters and $O\left(\frac{M^2}{\epsilon^2} \right)$ SGD iterations. We note that the network size is optimal up to the dependency on $\epsilon$ and poly-log factors, and the number of iteration is optimal up to a constant factor.
This result is valid for most Lipschitz losses including the hinge loss and the log-loss, and for most popular activation functions, including the ReLU.

\paragraph*{Technical Contribution}
For weights $\left(W,\bu\right)$ and $\x\in\sphere^{d-1}$ we denote by $\Psi_{W,\bu}(\x)\in\reals^{2q\times d}$ the gradient, w.r.t. the hidden weights $W$, of $h_{W,\bu}(\x)$. 
Our initialization scheme ensures that the SGD on the network, at the onset of the initialization process, is approximately equivalent to linear SGD starting at $0$, on top of the embedding $\Psi_{W,\bu}$, where $(W,\bu)$ are the initial weights. Now, it holds that 
\[
k(\x,\y) = \E_{W}\left[\frac{\inner{\Psi_{W,\bu}(\x), \Psi_{W,\bu}(\x)} }{2qB}\right]
\]
where $k$ is the kernel defined in \eqref{eq:ntk_intro}. Hence, if the network is large enough, we would expect that $k(\x,\y) \approx \frac{\inner{\Psi_{W,\bu}(\x), \Psi_{W,\bu}(\x)} }{2qB}$, and therefore that SGD on the network, in the onset of the initialization process, is approximately equivalent to linear SGD starting at $0$, w.r.t. the kernel $k$.
Our main technical contribution is the analysis of the rate (it terms of the size of the network) in which  $\frac{\inner{\Psi_{W,\bu}(\x), \Psi_{W,\bu}(\x)} }{2qB}$ converges to $k(\x,\y)$.

We would like to mention \citet{fiat2019decoupling} whose result shares some ideas with our proof. In their paper it 
is shown that for the square loss and the ReLU activation, linear optimization over the embedding $\Psi_{W,\bu}$ can 
memorize $m$ points with $\tilde{O}(m)$ parameters.

\section{Preliminaries}

\subsection{Notation} 
We denote
vectors by bold-face letters (e.g.\ $\x$), 
matrices by upper case letters (e.g.\ $W$), and collection of matrices by bold-face upper case letters (e.g.\ $\W$). 
We denote the $i$'s row in a matrix $W$ by $\w_i$.
The $p$-norm of $\x
\in \reals^d$ is denoted by $\|\x\|_p = \left(\sum_{i=1}^d|x_i|^p\right)^{\frac{1}{p}}$, and for a matrix $W$, $\|W\|$ is the spectral norm $\|W\| = \max_{\|\x\|=1}\|W\x\| $. We will also use the convention that $\|\x\|=\|\x\|_2$.
For a distribution $\cd$ on a space $\cx$, $p\ge 1$ and $f:\cx\to\reals$ we denote $\|f\|_{p,\cd} = \left(\E_{x\sim\cd}|f(x)|^p\right)^{\frac{1}{p}}$.
We use $\tilde{O}$ to hide poly-log factors.

\subsection{Supervised learning}
The goal in supervised learning is to
devise a mapping from the input space $\cx$ to an output space $\cy$ based on a sample
$S=\{(\x_1,y_1),\ldots,(\x_m,y_m)\}$, where $(\x_i,y_i)\in\cx\times\cy$ 
drawn i.i.d.\ from a distribution $\cd$ over $\cx\times\cy$. In our case, the instance space will always be $\sphere^{d-1}$.
A supervised learning problem is further specified by a loss function
$\ell : \reals \times \cy \to [0,\infty)$, and the goal is to find a
predictor $h:\cx\to\reals$ whose loss,
$\cl_{\cd}(h) := \E_{(\x,y)\sim\cd} \ell(h(\x),y)$, is small.
The {\em empirical} loss $\cl_{S}(h):= \frac 1 m \sum_{i=1}^m
\ell(h(\x_i),y_i)$ is commonly used as a proxy for the loss
$\cl_{\cd}$. 
When $h$ is defined by a vector $\w$ of parameters, we will use the notations $\cl_\cd(\w) = \cl_\cd(h)$, $\cl_S(\w)=\cl_S(h)$ and $\ell_{(\x,y)}(\w)=\ell(h(\x),y)$. For a class $\ch$ of predictors from $\cx$ to $\reals$ we denote $\cl_\cd(\ch) = \inf_{h\in\ch}\cl_\cd(h)$ and $\cl_S(\ch) = \inf_{h\in\ch}\cl_S(h)$

Regression problems correspond to $\cy=\reals$ and, for
instance, the squared loss $\ell^{\mathrm{square}}(\hat y,y)=(\hat y -y)^2$.
Classification is captured by $\cy=\{\pm 1\}$ and, say, the
zero-one loss $\ell^{0-1}(\hat y,y)= \ind[\hat y y \leq 0]$ or the hinge
loss $\ell^\hinge(\hat y,y)=[1-\hat y y]_+$.
A loss $\ell$ is $L$-Lipschitz if for all $y\in\cy$, the function $\ell_y(\hat y) :=\ell(\hat y,y)$ is $L$-Lipschitz. Likewise, it is
convex if $\ell_y$ is convex for every $y\in\cy$.
We say that $\ell$ is {\em $L$-decent} if for every $y\in\cy$, $\ell_y$ is  convex, $L$-Lipschitz, and twice differentiable in all but finitely many points.

\subsection{Neural network learning} 
We will consider fully connected neural networks of depth $2$ with $2q$ hidden neurons and activation function $\sigma:\reals\to\reals$. Throughout, we assume that the activation function is continuous, is twice differentiable in all but finitely many points, and there is $M>0$ such that $|\sigma'(x)|,|\sigma''(x)| \le M$ for every point $x\in\reals$ for which $f$ is twice differentiable in $x$. We call such an activation a {\em decent} activation. This includes most popular activations, including the ReLU activation $\sigma(x) = \max(0,x)$, as well as most sigmoids.

Denote 
\[
\cn^\sigma _{d,q} = \left\{h_{\W}(\x) = \inner{\bu, \sigma(W\x)} : W\in M_{2q,d}, \bu \in \reals^{2q}  \right\}~.
\]
We also denote by $\W = (W,\bu)$ the aggregation of all weights.
We next describe the learning algorithm that we analyze in this paper. 
We will use a variant of the popular Xavier initialization~\citep{glorot2010understanding} for the network weights, which we call {\em Xavier initialization with zero outputs}. 
The neurons will be arranged in pairs, where each pair consists of two neurons that are initialized identically, up to sign. Concretely, the weight 
matrix $W$ will be initialized to be a duplication $W = \begin{bmatrix}W'\\W'\end{bmatrix}$ of a matrix  $W'$ of standard Gaussians\footnote{It is 
more standard to assume that the instances has $L^2$ norm $O\left(\sqrt{d}\right)$, or infinity norm $O(1)$, and the entries of $W'$ has variance $\frac{1}{d}$. For the sake of notational convenience we chose a different scaling---divided the instances by $\sqrt{d}$ and accordingly multiplied the initial matrix by $\sqrt{d}$. Identical results can be derived for the more standard convention.} and $\bu$ will be a duplication of the all-$B$ vector in dimension $q$, for some $B>0$,  with its negation.
We denote the distribution of this initialization scheme by $\ci(d,q,B)$. Note that if $\W\sim \ci(d,q,B)$ then w.p. 1, $\forall \x,\;h_\W(\x)=0$. Finally, the training algorithm is described in \ref{alg:general_nn_training}.

\begin{algorithm}[ht]
	\caption{Neural Network Training}
	\begin{algorithmic}\label{alg:general_nn_training}
		\STATE \textbf{Input: } Network parameters $\sigma$ and $d,q$, loss $\ell$, initialization parameter $B>0$, learning rate $\eta>0$, batch size $b$, number of steps $T>0$, access to samples from a distribution $\cd$
		\STATE Sample  $\W^1\sim \ci(d,q,B)$
		\FOR {$t=1,\ldots,T$}
		\STATE Obtain a mini-batch $S_t=\{(\x^t_i,y^t_i)\}_{i=1}^b\sim\cd^b$
		\STATE With back-propagation, calculate a stochastic gradient $\nabla \cl_{S_t}(\W^t)$ and update $\W^{t+1} = \W^t - \eta \nabla \cl_{S_t}(\W^t)$
		\ENDFOR
		\STATE Choose $t\in [T]$ uniformly at random and return $\W_t$
	\end{algorithmic}
\end{algorithm}

\subsection{Kernel spaces}
Let $\cx$ be a set.
A {\em kernel} is a function
$k:\cx\times \cx\to\reals$ such that for every $x_1,\ldots,x_m\in \cx$ the
matrix $\{k(x_i,x_j)\}_{i,j}$ is positive semi-definite. 
A {\em kernel space} is a Hilbert space $\ch$ of functions from
$\cx$ to $\reals$ such that for every $x\in \cx$ the linear functional
$f\in\ch\mapsto f(x)$ is bounded.  The following theorem describes a
one-to-one correspondence between kernels and kernel spaces.

\begin{theorem}\label{thm:ker_spaces}
For every kernel $k$ there exists a unique kernel space $\ch_k$ such that for
every $x,x'\in \cx$, $k(x,x') = \inner{k(\cdot,x),k(\cdot,x')}_{\ch_k}$.
Likewise, for every kernel space $\ch$ there is a kernel $k$ for which
$\ch=\ch_k$.
\end{theorem}
\noindent
We denote the norm and inner product in $\ch_k$ by $\|\cdot\|_k$ and $\inner{\cdot,\cdot}_k$.
The following theorem describes
a tight connection between kernels and embeddings of $X$ into Hilbert spaces.
\begin{theorem}\label{thm:rkhs_embed}
A function $k:\cx \times \cx \to \reals$ is a kernel if and only if there
exists a mapping $\Psi: \cx \to\ch$ to some Hilbert space for which
$k(x,x')=\inner{\Psi(x),\Psi(x')}_{\ch}$. In this case,
$ \ch_k = \{f_{\Psi,\bv} \mid \mathbf{v}\in\ch \} $
where $f_{\Psi,\bv}(x) = \inner{\mathbf{v},\Psi(x)}_\ch$. Furthermore,
$\|f\|_{k} = \min\{\|\mathbf{v}\|_\ch : f_{\Psi,\bv}\}$ and
the minimizer is unique.
\end{theorem}

A special type of kernels that we will useful for us are {\em inner product kernels}. These are kernels $k:\sphere^{d-1}\times\sphere^{d-1}\reals$ of the form 
\[
k(\x,\y) = \sum_{n=0}^\infty b_n\inner{\x,\y}^n
\]
For scalars $b_n\ge 0$ with $\sum_{n=0}^\infty b_n < \infty$. It is well known that for any such sequence $k$ is a kernel. The following lemma summarizes a few properties of inner product kernels. 
\begin{lemma}\label{lem:inner_prod_ker}
Let $k$ be the inner product kernel $k(\x,\y) = \sum_{n=0}^\infty b_n\inner{\x,\y}^n$. Suppose that $b_n > 0$
\begin{enumerate}
\item 
If  $p(\x) = \sum_{|\alpha | =n }a_\alpha\x^{\alpha}$ then $p\in \ch_k$ and furthermore $\|p\|^2_k \le \frac{1}{b_n} \sum_{|\alpha | =n }a^2_\alpha$
\item For every $\bu\in\sphere^{d-1}$, the function $f(\x) = \inner{\bu,\x}^n$ belongs to $\ch_k$ and $\|f\|^2_k = \frac{1}{b_n}$
\end{enumerate}
\end{lemma}
For a kernel $k$ and $M>0$ we denote $\ch_k^M = \{h\in\ch_k : \|h\|_k\le M\}$. We note that spaces of the form $\ch_k^M$ often form a benchmark for learning algorithms.

\subsection{Hermite Polynomials and the dual activation}
Hermite polynomials $h_0,h_1,h_2,\ldots$ are the sequence of orthonormal polynomials corresponding to the standard Gaussian measure on $\reals$. Fix an activation $\sigma:\reals\to\reals$.  Following the terminology of \cite{daniely2016toward} we define the {\em dual activation} of $\sigma$ as
\[
\hat \sigma(\rho) = \E_{X,Y\text{ are $\rho$-correlated standard Gaussian}}\sigma(X)\sigma(Y)
\]
It holds that if $\sigma = \sum_{n=0}^\infty a_nh_n$ then
\[
\hat \sigma(\rho) = \sum_{n=0}^\infty a_n^2\rho^n
\]
In particular, $k_\sigma(\x,\y):=\hat\sigma\left(\inner{\x,\y}\right)$ is an inner product kernel.

\subsection{The Neural Tangent Kernel}\label{sec:ntk}
Fix network parameters $\sigma, d,q$ and $B$.
The {\em neural tangent kernel} corresponding to weights $\W$ is\footnote{The division by $2qB^2$ is for notational convenience.}
\[
\tk_\W(\x,\y) = \frac{\inner{ \nabla_\W h_{\W}(\x), \nabla_\W h_{\W}(\y)}}{2qB^2}
\]
The neural tangent kernel space, $\ch_{\tk_\W}$, is a linear approximation of the trajectories in which $h_W$ changes by changing $W$ a bit. 
Specifically, $h\in \ch_{\tk_\W}$ if and only if there is $\U$  such that
\begin{equation}\label{eq:ntk_description}
\forall \x \in \sphere^{d_1-1},\;\; h(\x) =  \lim_{\epsilon\to 0}\frac{h_{\W+\epsilon \U}(\x)  - h_{\W}(\x)}{\epsilon}
\end{equation}
Furthermore, we have that $\sqrt{q}B\cdot \|h\|_{\tk_\W}$ is the minimal Euclidean norm of $\U$ that satisfies equation \eqref{eq:ntk_description}.
The {\em expected initial neural tangent kernel}  is 
\[
\tk_{\sigma, B}(\x,\y) = \tk_{\sigma,d,q, B}(\x,\y) = \E_{\W\sim(d,q,B)} \tk_\W(\x,\y) 
\]
We will later see that $\tk_{\sigma,d,q, B}$ depends only on $\sigma$ and $B$.
If the network is large enough, we can expect that at the onset of the optimization process, $\tk_{\sigma,B}\approx k_\W$. 
Hence, approximately, $\ch_{\tk_{\sigma,B}}$ consists of the directions in which the initial function computed by the network can 
move. Since the initial function (according to Xavier initialization with zero outputs) is $0$, $\ch_{\tk_{\sigma,B}}$ is a 
linear approximation of the space of functions computed by the network in the vicinity of the initial weights. NTK theory 
based of the fact close enough to the initialization point, the linear approximation is good, and hence SGD on NN can learn functions 
in $\ch_{\tk_{\sigma,B}}$ that has sufficiently small norm. The main question is how small should the norm be, or 
alternatively, how large should the network be.

We next derive a formula for $\tk_{\sigma,B}$. 
We have, for $\W\sim\ci(d,q,B)$
\begin{eqnarray*}
\tk_\W(\x,\y) &=& \frac{\inner{ \nabla_W h_{W}(\x), \nabla_W h_{W}(\y)}}{2qB^2}
\\
&=& \frac{1}{qB^2}\sum_{i=1}^{q} \inner{B\sigma'\left( \inner{\w_{i},\x}  \right)\x ,B\sigma'\left( \inner{\w_{i},\y}  \right)\y} +
\frac{1}{qB^2}\sum_{i=1}^{q} \sigma\left( \inner{\w_{i},\x}  \right) \sigma\left( \inner{\w_{i},\y}  \right)
\\
&=& \frac{\inner{\x,\y}}{q}\sum_{i=1}^{q} \sigma'\left( \inner{\w_{i},\x}\right) \sigma'\left( \inner{\w_{i},\y}\right) +
\frac{1}{qB^2}\sum_{i=1}^{q} \sigma\left( \inner{\w_{i},\x}  \right) \sigma\left( \inner{\w_{i},\y}  \right)
\end{eqnarray*}
Taking expectation  we get
\[
\tk_{\sigma,B}(\x,\y) =  \inner{\x,\y}\hat\sigma'\left(\inner{\x,\y}\right) + \frac{1}{B^2}\hat\sigma\left(\inner{\x,\y}\right) =  \inner{\x,\y}k_{\sigma'}(\x,\y) + \frac{1}{B^2}k_\sigma(\x,\y)
\]
Finally, we decompose the expected initial neural tangent kernel into two kernels, that corresponds to the hidden and output weights respectively. Namely, we let 
\[
\tk_{\sigma,B} = \tk^h_{\sigma,B} + \tk^o_{\sigma,B} \text{ for } \tk^h_{\sigma}(\x,\y) =  \inner{\x,\y}\hat\sigma'\left(\inner{\x,\y}\right) \text{ and } \tk^o_{\sigma,B}(\x,\y) = \frac{1}{B^2}\hat\sigma\left(\inner{\x,\y}\right) 
\]

\section{Results}

\subsection{Learning the neural tangent kernel space with SGD on NN}

Fix a decent activation function $\sigma$ and  a decent loss $\ell$.  
Our first result shows that algorithm \ref{alg:general_nn_training} can learn the class $\ch^M_{\tk^h_\sigma}$ using a network with $\tilde{O}\left(\frac{M^2}{\epsilon^2}\right)$ parameters and using  $O\left(\frac{M^2}{\epsilon^2}\right)$ examples. We note that unless $\sigma$ is linear, the number of samples is optimal up to constant factor, and the number of parameters is optimal, up to poly-log factor and the dependency on $\epsilon$. This remains true even if we restrict to $O(1)$-bounded distributions.

\begin{theorem}\label{thm:learning_conj_class}
Given $d$, $M>0$, $R>0$ and $\epsilon>0$ there is a choice of $q = \tilde{O}\left(\frac{M^2R^2}{d\epsilon^2}\right)$, $T=O\left(\frac{M^2}{\epsilon^2}\right)$, as well as $B>0$ and $\eta>0$, such that for every $R$-bounded distribution $\cd$ and batch size $b$, the function $h$ returned by algorithm \ref{alg:general_nn_training} satisfies $\E \cl_\cd(h) \le \cl_\cd\left(\ch_{\tk_\sigma^h}^M\right) + \epsilon$
\end{theorem}

As an application, we conclude that for the ReLU activation, algorithm \ref{alg:general_nn_training} can learn even polynomials of bounded norm with near optimal sample complexity and network size.
We denote 
\[
\cp^M_{c} = \left\{ p(\x) = \sum_{|\alpha | \text{ is even and }\le c}a_\alpha\x^{\alpha} : \sum_{|\alpha | \text{ is even and }\le c}a^2_\alpha \le M^2 \right\}
\]

\begin{theorem}\label{thm:learning_pol}
Fix a constant $c>0$ and assume that the activation is ReLU.
Given $d$, $M>0$, $R>0$ and $\epsilon>0$ there is a choice of $q = \tilde{O}\left(\frac{M^2R^2}{d\epsilon^2}\right)$, $T=O\left(\frac{M^2}{\epsilon^2}\right)$, as well as $B>0$ and $\eta>0$, such that for every $R$-bounded distribution $\cd$ and batch size $b$, the function $h$ returned by algorithm \ref{alg:general_nn_training} satisfies $\E \cl_\cd(h) \le \cl_\cd\left(\cp_{c}^M\right) + \epsilon$
\end{theorem}
We note that as in theorem \ref{thm:learning_conj_class}, the number of samples is optimal up to constant factor, and the number of parameters is optimal, up to poly-log factor and the dependency on $\epsilon$,
and this remains true even if we restrict to $O(1)$-bounded distributions.

\subsection{Memorization}

Theorem \ref{thm:learning_conj_class} can be applied to analyze  memorization by SGD. Assume that $\ell$ is the hinge loss (similar result is valid for many other losses such as the log-loss) and $\sigma$ is any decent non-linear activation. Let $S = \{(\x_1,y_1),\ldots,(\x_m,y_m)\}$ be $m$ random, independent and uniform points in $\sphere^{d-1}\times\{\pm 1\}$ with $m=d^c$ for some $c>1$. Suppose that we run SGD on top of $S$. Namely, we run algorithm \ref{alg:general_nn_training} where the underlying distribution is the uniform distribution on the points in $S$. Let $h:\sphere^{d-1}\to\reals$ be the output of the algorithm.
We say that the algorithm memorized the $i$'th example if $y_ih(\x_i) > 0$.
The memorization problem investigate how many points the algorithm can memorize, were most of the focus is on how large the network should be in order to memorize $1-\epsilon$ fraction of the points.

As shown in section \ref{sec:bounded}, the uniform distribution on the examples in $S$ is $\left(1+o(1)\right)$-bounded w.h.p. over the choice of $S$. Likewise, it is not hard to show that w.h.p. over the choice of $S$ there is a function $h^*\in \ch_k^{O(m)}$ such that $h^*(\x_i) = y_i$ for all $i$.
By theorem \ref{thm:learning_conj_class} we can conclude the by running SGD on a network with $\tilde{O}(\frac{m}{\epsilon^2})$ parameters and $O\left(\frac{m}{\epsilon^2}\right)$ steps, the network will memorize $1-\epsilon$ fraction of the points. This size of networks is optimal up to poly-log factors, and the dependency of $\epsilon$. This is satisfactory is $\epsilon$ is considered a constant.
However, for small $\epsilon$, more can be desired. For instance, in the case that we want to memorize all points,  we need $\epsilon < \frac{1}{m}$, and we get a network with $m^3$ parameters. To circumvent that, we perofem a more refined analysis of this memorization problem and show that even perfect memorization of $m$ points can be done via SGD on a network with $\tilde{O}(m)$ parameters, which is optimal, up to poly-log factors.

\begin{theorem}\label{thm:memorization_main}
There is a choice of $q = \tilde{O}\left(\frac{m}{d}\right)$, $T=\tilde{O}\left(\frac{m}{\epsilon^2} \right)$, as well as $B>0$ and $\eta>0$, such that for every batch size $b$, w.p. $1-o_m(1)$, the function $h$ returned by algorithm \ref{alg:general_nn_training} memorizes $1-\epsilon$ fraction of the  examples.
\end{theorem}

We emphasize the our result is true for any non-linear and decent activation function.

\subsection{Open Questions}

The most obvious open question is to generalize our results to the standard Xavier initialization, where $W$ is a matrix of independent Gaussians of variance $\frac{1}{d}$, while $\bu$ is a vector of independent Gaussians of variance $\frac{1}{q}$. Another open question is to generalize our result to deeper networks.

\section{Proofs}

\subsection{Reduction to SGD over vector random features}

We will prove our result via a reduction to linear learning over the initial neural tangent kernel space, corresponding the the hidden weights. 

That is, we define by $\Psi_\W(\x)$ the gradient of the function $\W\mapsto h_\W(\x)$ w.r.t. the hidden weights. Namely,
\[
\Psi_\W(\x) = \left(u_1\sigma'(\inner{\w_1,\x})\x,\ldots,u_{2q}\sigma'(\inner{\w_{2q},\x})\x\right) \in \reals^{2q\times d}
\] 
Denote $f_{\Psi_\W,\V} (\x) = \inner{\V,\Psi_\W(\x)}$ and consider algorithm \ref{alg:ntk_training}.
\begin{algorithm}[ht]
	\caption{Neural Tangent Kernel Training}
	\begin{algorithmic}\label{alg:ntk_training}
		\STATE \textbf{Input: } Network parameters $\sigma$ and $d,q$, loss $\ell$, learning rate $\eta>0$, batch size $b$, number of steps $T>0$, access to samples from a distribution $\cd$
		\STATE Sample  $\W\sim \ci(d,q,1)$
		\STATE Initialize $\V^1 = 0\in \reals^{2q\times d}$
		\FOR {$t=1,\ldots,T$}
		\STATE Obtain a mini-batch $S_t=\{(\x^t_i,y^t_i)\}_{i=1}^b\sim\cd^b$
		\STATE Using back-propagation, calculate the gradient $\nabla$ of $\cl_{S_t}(\V) =  \cl_{S_t}\left(f_{\Psi_\W,\V}  \right)$ at $\V^t$
		\STATE Update $\V^{t+1} = \V^t - \eta \nabla$
		\ENDFOR
		\STATE Choose $t\in [T]$ uniformly at random and return $f_{ \Psi_W,\V_t}$
	\end{algorithmic}
\end{algorithm}

It is not hard to show that by taking  large enough $B$, algorithm \ref{alg:general_nn_training} is essentially equivalent to algorithm \ref{alg:ntk_training}. Namely,

\begin{lemma}\label{lem:alg_equivalence}
Fix a decent activation $\sigma$ as well as convex a decent loss $\ell$.
There is a choice $B = poly(d,q,1/\eta,T,1/\epsilon)$, such that for every input distribution the following holds. Let $h_1,h_2$ be the functions returned algorithm \ref{alg:general_nn_training} with parameters $d,q,\frac{\eta}{B^2},b,B,T$ and algorithm \ref{alg:ntk_training} with parameters $d,q,\eta,b,T$. Then, $|\E\cl_\cd(h_1)-\E\cl_\cd(h_2)|<\epsilon$
\end{lemma} 

By lemma \ref{lem:alg_equivalence} in order to prove theorem \ref{thm:learning_conj_class} it is enough to analyze algorithm \ref{alg:ntk_training}. Specifically, theorem \ref{thm:learning_conj_class} follows form the following theorem:

\begin{theorem}\label{thm:ntk_learning}
Given $d$, $M>0$, $R>0$ and $\epsilon>0$ there is a choice of $q = \tilde{O}\left(\frac{M^2R^2}{d\epsilon^2}\right)$, $T=O\left(\frac{M^2}{\epsilon^2}\right)$, as well as $\eta>0$, such that for every $R$-bounded distribution $\cd$ and batch size $b$, the function $h$ returned by algorithm \ref{alg:ntk_training} satisfies $\E \cl_\cd(h) \le \cl_\cd\left(\ch_{\tk_\sigma^h}^M\right) + \epsilon$
\end{theorem}

Our next step is to rephrase algorithm \ref{alg:ntk_training} in the language of (vector) random features.
We note that algorithm \ref{alg:ntk_training} is SGD on top of the random embedding $\Psi_\W$. This embedding composed of $q$ i.i.d. random mappings $\psi_{\w}(\x) = \left(\sigma'(\inner{\w,\x})\x, -\sigma'(\inner{\w,\x})\x\right)$ where $\w\in\reals^{d}$ is a standard Gaussian. This can be slightly simplified to SGD on top of the i.i.d. random mappings $\psi_{\w}(\x) = \sigma'(\inner{\w,\x})\x$. Indeed, if we make this change the inner products between the different examples, after the mapping is applied, do not change (up to multiplication by $\sqrt{2}$), and SGD only depends on these inner products.
This falls in the framework of learning with (vector) random features scheme, which we define next, and analyze in the next section.

Let $\cx$ be a measurable space and let $k:\cx\times\cx\to \reals$ be a kernel.  A {\em random features scheme} (RFS) for $k$ is a pair
$(\psi,\mu)$ where $\mu$ is a probability measure on a measurable space
$\Omega$, and $\psi:\Omega\times\cx\to \reals^d$ is a measurable function,
such that
\begin{equation}\label{eq:ker_eq_inner}
\forall \x,\x'\in\cx,\;\;\;\;k(\x,\x') =
  \E_{\omega\sim \mu}\left[\inner{\psi(\omega,\x),\psi(\omega,\x')}\right]\,.
\end{equation}
We often refer to $\psi$ (rather than $(\psi,\mu)$) as the RFS.
The {\em NTK RFS} is given by the mapping $\psi:\reals^d\times\sphere^{d-1}\to\reals^{d}$ defined by
\[
\psi(\omega,\x) = \sigma'(\inner{\omega,\x})\x
\]
an $\mu$ being the standard Gaussian measure on $\reals^d$. It is an RFS for the kernel $\tk_\sigma^h$ (see section \ref{sec:ntk}).
We define the {\em norm} of $\psi$ as $\|\psi\| = \sup_{\omega,\x}|\psi(\omega,\x)|$. We say that $\psi$ is {\em $C$-bounded} if $\|\psi\|\le C$. We note that the NTK RFS is $C$-bounded for $C = \| \sigma'\|_\infty$.
We say that an RFS $\psi:\Omega\times \sphere^{d-1}\to\reals^d$ is {\em factorized} if there is a function  $\psi':\Omega\times \sphere^{d-1}\to\reals$ such that $\psi(\omega,\x) = \psi'(\omega,\x) \x$. We note that the NTK RFS is factorized.

A random {\em $q$-embedding} generated from $\psi$ is the random mapping
\[
\Psi_\vomega(\x) \eqdef
\frac{\left(\psi({\omega_1},\x),\ldots , \psi({\omega_q},\x) \right)}
     {\sqrt{q}} \,,
\]
where $\omega_1,\ldots,\omega_q\sim \mu$ are i.i.d. 
We next consider an algorithm for learning $\ch_k$, by running SGD on top of random features.

\begin{algorithm}[H]
	\caption{SGD on RFS}
	\begin{algorithmic}\label{alg:rfs_training}
		\STATE \textbf{Input: } RFS $\psi:\Omega\times\cx\to\reals^d$, number of random features $q$, loss $\ell$, learning rate $\eta>0$, batch size $b$, number of steps $T>0$, access to samples from a distribution $\cd$
		\STATE Sample  $\vomega\sim \mu^q$
		\STATE Initialize $\bv^1 = 0\in \reals^{q\times d}$
		\FOR {$t=1,\ldots,T$}
		\STATE Obtain a mini-batch $S_t=\{(\x^t_i,y^t_i)\}_{i=1}^b\sim\cd^b$
		\STATE Update $\bv_{t+1} = \bv_t - \eta \nabla \cl_{S_t} \left( \bv_t \right)$ where $\cl_{S_t}(\bv) = \cl_{S_t}\left(f_{ \Psi_\vomega,\bv }  \right)$.
		\ENDFOR
		\STATE Choose $t\in [T]$ uniformly at random and return $f_{ \Psi_\vomega,\bv_t}$
	\end{algorithmic}
\end{algorithm}

\begin{theorem}\label{thm:sgd_on_rfs}
Assume that $\psi$ is factorized and $C$-bounded RFS for $k$, that $\ell$ is convex and $L$-Lipschitz, and that $\cd$ has $R$-bounded marginal. Let $f$ be the function returned by algorithm \ref{alg:rfs_training}. Fix a function $f^*\in\ch_k$. Then
\[
\E\cl_\cd(f) \le \cl_\cd(f^*) +  \frac{LRC\|f^*\|_k}{\sqrt{qd}} + \frac{\|f^*\|^2_k}{2\eta T} + \frac{\eta  L^2C^2}{2}
\]
In particular, if  $\|f^*\|_k\le M$ and $\eta = \frac{M}{\sqrt{T}LC}$ we have
\[
\E\cl_\cd(f) \le L_\cd(f^*) + \frac{LRCM}{\sqrt{qd}} + \frac{LCM}{\sqrt{T}}
\]
\end{theorem}
The next section is devoted to the analysis of RFS an in particular to the proof of theorem \ref{thm:sgd_on_rfs}. We note that since the NTK RFS  is factorized and $C$-bounded (for $C = \| \sigma'\|_\infty$), theorem \ref{thm:ntk_learning}   follows from theorem \ref{thm:sgd_on_rfs}. Together with lemma \ref{lem:alg_equivalence}, this implies theorem \ref{thm:learning_conj_class}.

\subsection{Vector random feature schemes}
\label{rfs:sec}
 For the rest of this section, let us fix a $C$-bounded RFS $\psi$ for a
kernel $k$ and a random $q$ embedding $\Psi_\vomega$.
The random {\em $q$-kernel} corresponding to $\Psi_\vomega$ is $k_\vomega(\x,\x') =
\inner{\Psi_\vomega(\x),\Psi_\vomega(\x')}$. Likewise, the random
{\em $q$-kernel space} corresponding to $\Psi_\vomega$ is
$\ch_{k_\vomega}$.
For every $\x,\x'\in \cx$
\[
k_\vomega(\x,\x') =
  \frac{1}{q}\sum_{i=1}^q \inner{\psi({\omega_i},\x),\psi({\omega_i},\x')}
\]
is an average of $q$ independent random variables whose expectation is
$k(\x,\x')$. By Hoeffding's bound we have.
\begin{theorem}[Kernel Approximation]\label{thm:ker_apx}
Assume that $q\ge \frac{2C^4\log\left(\frac{2}{\delta}\right)}{\epsilon^2}$,
then for every $\x, \x'\in \cx$ we have
$\Pr\left(\left| k_\vomega(\x,\x') - k(\x,\x') \right| \ge \epsilon \right)
  \le \delta$.
\end{theorem}
\noindent We next discuss approximation of functions in $\ch_k$ by functions in
$\ch_{k_\vomega}$. It would be useful to consider the embedding
\begin{equation}\label{eqn:psi-embedding}
\x\mapsto\Psi^\x \; \mbox{ where } \;
  \Psi^\x\eqdef\psi(\cdot,\x)\in L^2(\Omega, \reals^d) \,.
\end{equation}
From~\eqref{eq:ker_eq_inner} it holds
that for any $\x,\x'\in\cx$,
$k(\x,\x') = \inner{\Psi^\x,\Psi^{\x'}}_{L^2(\Omega)}$.
In particular, from Theorem~\ref{thm:rkhs_embed}, for every $f\in\ch_k$ there
is a unique function $\check{f}\in L^2(\Omega,\reals^d)$ such that
\begin{equation}\label{eq:f_norm_eq}
\|\check{f}\|_{L^2(\Omega)} = \|f\|_{k}
\end{equation}
and for every $\x\in\cx$,
\begin{equation}\label{eq:f_x_as_inner}
f(\x) = \inner{\check{f},\Psi^\x}_{L^2(\Omega, \reals^d)} =  \E_{\omega\sim\mu}\inner{\check{f}(\omega),\psi(\omega,\x)}\,.
\end{equation}

\begin{example}\label{example:f_check}
Fix $\sigma:\reals\to\reals$ with Hermite expansion $\sigma = \sum_{n=0}^\infty a_nh_n$
and let $\Omega=\reals^d$ and $\cx=\sphere^{d-1}$
\begin{enumerate}
\item
Consider the RFS  $\psi(\omega,\x) = \sigma\left(\inner{\omega,\x}\right)$ with $\mu$ begin the standard Gaussian measure on $\reals^d$.
We have that $\psi$ is an RFS for the kernel $k(\x,\y) = \hat \sigma\left(\inner{\x,\y}\right)$. Consider the function $f(\x) = \inner{\x_0,\x}^n$.
We claim that $\check{f}(\vomega) = \frac{1}{a_n}h_n\left(\inner{\x_0,\omega}\right)$.
Indeed, we have,
\begin{eqnarray*}
\E_{\omega\sim\mu} \sigma\left( \inner{\omega,\x} \right)\frac{1}{a_n}h_n\left(\inner{\x_0,\omega}\right) 
&=& \frac{1}{a_n} \sum_{k=0}^\infty\E_{\omega\sim\mu} a_k h_k\left(\inner{\omega,\x}\right)h_n\left(\inner{\x_0,\vomega}\right)
\\
&=& \frac{1}{a_n} \sum_{k=0}^\infty a_k\delta_{kn}\inner{\x,\x_0}^k
\\
&=& \inner{\x,\x_0}^n
\end{eqnarray*}
and
\[
\left\|\omega\mapsto \frac{1}{a_n}h_n\left(\inner{\x_0,\omega}\right)\right\|_{L^2(\Omega)} = \E_{\omega\sim\mu} \frac{1}{a_n^2}h^2_n\left(\inner{\x_0,\omega}\right) = \frac{1}{a_n^2} = \|f\|_k^2
\]
\item
Consider the NTK RFS  $\psi(\omega,\x) = \sigma\left(\inner{\omega,\x}\right)\x$ with $\mu$ begin the standard Gaussian measure on $\reals^d$.  We have that $\psi$ is an RFS for the kernel $k(\x,\y) = \inner{\x,\y}\hat \sigma\left(\inner{\x,\y}\right)$. Consider  the function $f(\x) = \left(\inner{\x_0,\x}\right)^n$. As in the item above, it is not hard to show that $\check{f}(\omega) = \frac{1}{a_{n-1}}h_{n-1}\left(\inner{\x_0,\omega}\right)\x_0$.
\end{enumerate}
\end{example}
Let us denote
$f_\vomega(\x) = \frac{1}{q}\sum_{i=1}^q
  \inner{\check{f}(\omega_i),\psi(\omega_i,\x)}$.
From~\eqref{eq:f_x_as_inner} we have that
$\E_\vomega\left[f_\vomega(\x)\right] = f(\x)$.
Furthermore, for every $\x$,
the variance of $f_\vomega(\x)$ is at most
\begin{eqnarray*}
\frac{1}{q}\E_{\omega\sim\mu}
  \left|\inner{\check{f}(\omega),\psi(\omega,\x)}\right|^2
&\le &
\frac{C^2}{q}\E_{\omega\sim\mu}
  \left|\check{f}(\omega)\right|^2 
\\
&=& \frac{C^2\|f\|^2_{k}}{q}\,.
\end{eqnarray*}
An immediate consequence is the following corollary.
\begin{corollary} [Function Approximation] \label{thm:func_apx}
For all $\x\in\cx$,
$\E_\vomega|f(\x) - f_\vomega(\x)|^2 \le \frac{C^2\|f\|^2_{k}}{q}$.
\end{corollary} \noindent
Now, if $\cd$ is a distribution on $\cx$ we get that 
\[
\E_\vomega\|f - f_\vomega\|_{2,\cd} \stackrel{\text{Jensen}}{\le}  \sqrt{\E_\vomega\|f - f_\vomega\|^2_{2,\cd}}   = \sqrt{\E_\vomega\E_{\x\sim\cd}|f(\x) - f_\vomega(\x)|^2} = \sqrt{\E_\x\E_\vomega|f(\x) - f_\vomega(\x)|^2} \le \frac{C\|f\|_{k}}{\sqrt{q}}
\]
Thus, $O\left(\frac{\|f\|_k^2}{\epsilon^2}\right)$ random features suffices to guarantee expected $L^2$ distance of at most $\epsilon$. Note the this bound does not depend on $d$, the dimension of a single random feature. We might expect that at least in some cases, $d$-dimensional random feature is as good as $d$ one-dimensional random features. We next describe a scenario in which this is true, and in particular $O\left(\frac{\|f\|_k^2}{d\epsilon^2}\right)$ random features suffices to guarantee expected $L^2$ distance of at most $\epsilon$.

\begin{lemma}\label{lem:fun_apx}
Assume that $\psi:\Omega\times \sphere^{d-1}\to\reals^d$ is factorized and $\cd$ is $R$-bounded distribution. Then,
\[
\E_\vomega \|f-f_\vomega\|_{2,\cd} \le \sqrt{\E_\vomega \|f-f_\vomega\|^2_{2,\cd}} \le \frac{RC \|f\|_{k}}{\sqrt{qd}}
\]
Furthermore, if $\ell:\sphere^{d-1}\times Y\to [0,\infty)$, is $L$-Lipschitz loss and $\cd'$ is a distribution of $\sphere^{d-1}\times Y$ with $R$-bounded marginal then
\[
\E_{\vomega}\cl_{\cd'}(f_\vomega)\le   \cl_{\cd'}(f) + \frac{LRC\|f\|_{k}}{\sqrt{qd}}
\]
\end{lemma}
\begin{proof}
Let $\x\sim\cd$ and $\omega\sim\mu$. We have
\begin{eqnarray*}
\E_\vomega \|f-f_\vomega\|_{2,\cd}
& \stackrel{\text{Jensen's Inequality}}{\le} & \sqrt{\E_\vomega\|f-f_\vomega\|_{2,\cd}^2} 
\\
& =  & \sqrt{ \E_\vomega\E_\x  |f(\x) - f_\vomega(\x) |^2}\\
\\
& =  & \sqrt{\E_\x \E_\vomega  |f(\x) - f_\vomega(\x) |^2}\\
& =  & \sqrt{\frac{\E_\x \E_{\omega\sim\mu}  \left|\inner{\check{f}(\omega),\psi(\omega,\x)} - f(\x)\right|^2}{q}}\\
& \stackrel{\text{Variance is bounded by squared $L^2$ norm}}{\le}  & \sqrt{\frac{\E_\x \E_{\omega\sim\mu}  \left|\inner{\check{f}(\omega),\psi(\omega,\x)} \right|^2}{q}}\\
& =  & \sqrt{\frac{\E_{\omega\sim\mu} \E_\x   \left|\inner{\check{f}(\omega),\psi'(\omega,\x)\x} \right|^2}{q}}\\
& \stackrel{\psi\text{ and hence also }\psi'\text{ is $C$-bounded}}{\le}  & C \sqrt{\frac{\E_{\omega\sim\mu} \E_\x   \left|\inner{\check{f}(\omega),\x} \right|^2}{q}}\\
& \stackrel{\cd\text{ is $R$-bounded}}\le  & CR \sqrt{\frac{\E_{\omega\sim\mu}   \left\|\check{f}(\omega)\right\|^2}{qd}}\\
&\stackrel{\text{Equation \eqref{eq:f_norm_eq}}}{=} &\frac{CR\|f\|_{k}}{\sqrt{qd}} \,.
\end{eqnarray*}
Finally, for $L$-Lipschitz $\ell$, and $(\x,y)\sim\cd'$ then
\begin{eqnarray*}
\E_\vomega L_{\cd'}(f_\vomega) &=& \E_\vomega \E_{\x,y} \ell(f_\vomega(\x),y)
\\
 &\le&  \E_\vomega \E_{\x,y}\ell(f(\x),y) + L\E_\vomega \E_\x \left| f(\x)-f_\vomega(\x)  \right|
\\
&=& \E_{\x,y}\ell(f(\x),y) + L\E_\vomega \E_\x \left| f(\x)-f_\vomega(\x)  \right|
\\
&=& \cl_{\cd'}(f) + L\E_\vomega \E_\x \left| f(\x)-f_\vomega(\x)  \right|
\\
&\stackrel{L^1\le L^2}{\le}&  \cl_{\cd'}(f) + L\E_\vomega \sqrt{\E_\x\left| f(\x)-f_\vomega(\x)  \right|^2}
\\
&\le&  \cl_{\cd'}(f) + \frac{LCR\|f\|_{k}}{\sqrt{qd}}
\end{eqnarray*}

\end{proof}

Finally, we are ready to prove theorem \ref{thm:sgd_on_rfs}
\begin{proof}(of theorem \ref{thm:sgd_on_rfs})
Denote by $\bv^*\in\reals^{dq}$ the vector
\[
v^*_i=\frac{1}{\sqrt{q}}\left(\check f^*(\omega_1),   \ldots    ,\check f^*(\omega_1)\right)
\]
By standard results on SGD (e.g. \cite{shalev2014understanding}) we have that given $\omega$,
\begin{eqnarray*}
\cl_\cd(f) &\le& \cl_\cd(f^*_\vomega) +  \frac{1}{2\eta T}\|\bv^*\|^2 + \frac{\eta L^2C^2}{2} 
\end{eqnarray*}
Taking expectation over the choice of $\vomega$ and using lemma \ref{lem:fun_apx} and equation \eqref{eq:f_norm_eq} we have
\begin{eqnarray*}
\cl_\cd(f) &\le& \cl_\cd(f^*) +  \frac{LRC\|f^*\|_k}{\sqrt{qd}} + \frac{\|f^*\|^2_k}{2\eta T} + \frac{\eta  L^2C^2}{2} 
\end{eqnarray*}
\end{proof}

\subsection{Memorization of random set of points -- proof of theorem \ref{thm:memorization_main}}

Consider the NTK RFS  $\psi(\omega,\x) = \sigma'\left(\inner{\omega,\x}\right)\x$ with $\mu$ begin the standard Gaussian measure on $\reals^d$.  Recall that $\psi$ is an RFS for the kernel $\tk^h_\sigma(\x,\y) = \inner{\x,\y}\hat \sigma'\left(\inner{\x,\y}\right)$. 
As in the proof of theorem \ref{thm:learning_conj_class}, it is enough to show that for $q =  \tilde O\left(\frac{m}{d}\right)= \tilde O\left(d^{c-1}\right)$, w.p. $1-o(1)$ over the choice of $S$ and $\vomega = (\omega_1,\ldots,\omega_q)$, there is $\bv\in \reals^{dq}$  such that
\begin{equation}\label{eq:search_for_v}
\inner{\bv,\Psi_\vomega(\x_i)} = y_i + o(1)\text{ for all $i$ and }\|\bv\|_2^2 =\tilde{ O}(m)
\end{equation}
Choose a constant integer $c'>4c + 2$ such that $a_{c'-1} \ne 0$. Such a constant exists since $\sigma$ is not a polynomial.
Define
\[
f(\x) = \sum_{i=1}^m y_i\left(\inner{\x_i,\x}\right)^{c'}
\]
\begin{lemma}\label{lem:f_is_good}
With probability $1-\delta$ we have that
\[
f(\x_i) = y_i + O\left(\frac{\log^{\frac{c'}{2}}\left(d/\delta\right)}{d}\right) \text{ for all $i$ and }\|f\|^2_{k_\sigma} =  O\left(m\right) + O\left(\frac{\log^{\frac{c'}{2}}\left(d/\delta\right)}{d}\right)
\]
\end{lemma}
\begin{proof}
W.p $1-\delta$ we have that $\inner{\x_i,\x_j} \le O\left(\sqrt{\frac{\log\left(m/\delta\right)}{d}}\right) = O\left(\sqrt{\frac{\log\left(d/\delta\right)}{d}}\right)$ for all $i,j\in [m]$. In this case we have that for any $i$
\[
f(\x_i) = y_i + O\left(m\left(\frac{ \log\left(d/\delta\right)}{d}\right)^{\frac{c'}{2}}\right) = y_i + O\left(\log^{\frac{c'}{2}}\left(d/\delta\right)d^{c-\frac{c'}{2}}\right)  = y_i + O\left(\frac{\log^{\frac{c'}{2}}\left(d/\delta\right)}{d}\right)
\]
Likewise,
\[
\|f\|^2_{k_\sigma} = a^{-2}_{c'}m + O\left(m^2\left(\frac{ \log\left(d/\delta\right)}{d}\right)^{\frac{c'}{2}}\right)  = a^{-2}_{c'}m + O\left(\log^{\frac{c'}{2}}\left(d/\delta\right)d^{2c-\frac{c'}{2}}\right) = a^{-2}_{c'}m + O\left(\frac{\log^{\frac{c'}{2}}\left(d/\delta\right)}{d}\right)
\]
\end{proof}
Based on lemma \ref{lem:f_is_good}, in order to find $\bv$ that satisfies equation \eqref{eq:search_for_v} it is natural to take 
\[
\bv=\frac{1}{\sqrt{q}}\left(\check f(\omega_1),   \ldots    ,\check f(\omega_q)\right)
\]
In which case $\E\|\bv\|_2^2 = \|f\|_{k_\sigma}^2$ and  $\E\left[\inner{\bv,\Psi_\vomega(\x)}\right] = \E\left[f_\vomega(\x)\right] = f(\x)$. In fact, lemma \ref{lem:fun_apx} together with Chebyshev's inequality  indeed implies that for large $q$ equation \eqref{eq:search_for_v} holds. However, this analysis requires $q\approx \frac{m^2}{d}$ while we want $q\approx \frac{m}{d}$. In the remaining part of this section we undertake a more delicate anlysis of the rate in which $f_\vomega$ approximates $f$ in our specific case. This analysis will imply that $q = \tilde{O}\left(\frac{m}{d}\right)$ suffices for equation \eqref{eq:search_for_v} to hold w.h.p. Indeed, we will prove that

\begin{lemma}\label{lem:exponential_tail_apx}
W.p. $1-\delta - 2^{\Omega(d)}$ over the choice of $S$ and $\vomega$, we have that
\[
\forall i\in [m],\;\;\;\;\left|f_\vomega(\x_i) - f(\x_i)\right| \le O\left(\sqrt{\frac{m\log^{c'+2}(m/\delta)}{dq}}\right)
\]
\end{lemma}

Togeter with lemma \ref{lem:f_is_good} and Markov's inequality we have
\begin{theorem}\label{thm:memorization}
W.p. $1-\delta - 2^{\Omega(d)}$ over the choice of $S$ and $\vomega$, we have that
\[
\inner{\bv,\Psi_\vomega(\x_i)} = f_\vomega(\x_i) = y_i + O\left(\frac{\log^{\frac{c'}{2}}\left(d/\delta\right)}{d}\right) + O\left(\sqrt{\frac{d^{c-1}\log^{c'+2}(d/\delta)}{q}}\right) \text{ for all $i$}
\]
and
\[
\|\bv\|^2_{2} =  O\left(m/\delta\right) + O\left(\frac{\log^{\frac{c'}{2}}\left(d/\delta\right)}{d\delta}\right)
\]
\end{theorem}
Choosing $\delta =\frac{1}{\log(m)}$ we get that for $q = \tilde{O}\left(d^{c-1}\right)$ equation \eqref{eq:search_for_v} holds w.p. $1-o(1)$. This proves theorem \ref{thm:memorization_main}.
The remaining part of the section is a proof of lemma \ref{lem:exponential_tail_apx}. We will need the following version of Hoeffding's bound.
A distribution $\mu$ on $\reals$ is called {\em $(\delta,B)$-bounded} if $\Pr_{X\sim\mu}(|X| >  B)\le \delta$.

\begin{lemma}\label{lem:hof_like}
Let $\mu$  be a $(\delta,B)$-bounded distribution and let $X_1,\ldots,X_m$ be i.i.d. r.v. from $\mu$. Then, w.p. $1-m\delta - \delta'$
\[
\left|\E_{X\sim\mu}[X] - \frac{1}{m}\sum_{i=1}^mX_i  \right| \le B\sqrt{\frac{2\ln(\delta'/2)}{m}} + \frac{2\sqrt{\delta\E_{X\sim\mu}X^2}}{1-\delta}
\]
\end{lemma}
\begin{proof}
We note that given that $X_i\in [-B,B]$ for all $i$ we have by Hoeffding's bound that w.p. $1-\delta'$
\[
\left|\frac{1}{m}\sum_{i=1}^mX_i - \E_{X\sim\mu}[X|X\in [-B,B]]  \right| \le B\sqrt{\frac{2\ln(\delta'/2)}{m}}
\]
We  note that 
\begin{eqnarray*}
\E_{X\sim\mu}[X|X\in[-B,B]]  &=& \frac{\E_{X\sim\mu}X + \delta \E_{X\sim\mu}[X|X\notin[-B,B]] }{1-\delta}
\\
&=& \frac{\E_{X\sim\mu}X + \E_{X\sim\mu}[X1[X\notin[-B,B]] }{1-\delta}
\end{eqnarray*}
Hence, by Cauchy-Schwartz,
\begin{eqnarray*}
\left|\E_{X\sim\mu}[X|X\in[-B,B]] - \E_{X\sim\mu}[X]  \right| \le  \frac{\delta}{1-\delta}\left|\E_{X\sim\mu}X\right| + \frac{\sqrt{\delta\E_{X\sim\mu}X^2}}{1-\delta} \le \frac{2\sqrt{\delta\E_{X\sim\mu}X^2}}{1-\delta}
\end{eqnarray*}
\end{proof}
Recall now that by example \ref{example:f_check}
\[
\check{f}(\omega) = \sum_{i=1}^m\frac{y_i}{a_{c'-1}}h_{c'-1}\left(\inner{\x_i,\omega}\right)\x_i
\]
Hence, for any $\x$,
\[
f_\vomega(\x) = \frac{1}{q}\sum_{j=1}^q\sum_{i=1}^m\frac{y_i}{a_{c'-1}}h_{c'-1}\left(\inner{\x_i,\omega_j}\right)\inner{\x_i,\x}\sigma\left(\inner{\omega_j,\x}\right)
\]
In particular, fixing $S$, $f_\vomega(\x)$ is an average of the $q$ i.i.d. random variables
\[
f_\vomega(\x) = \frac{1}{q}\sum_{j=1}^qY(\omega_i,\x)
\]
Where
\[
Y(\omega, \x) = \sum_{i=1}^m\frac{y_i}{a_{c'-1}}h_{c'-1}\left(\inner{\x_i,\omega}\right)\inner{\x_i,\x}\sigma\left(\inner{\omega,\x}\right)
\]

\begin{lemma}\label{lem:bouded_w}
W.p. $\ge 1-\delta$ over the choice of $S$, we have that for every $i\in [m]$, $Y(\omega,\x_i)$ is $\left(\delta + 2^{-\Omega(d)},O\left(\sqrt{\frac{m\log^{c'+1}(m/\delta)}{d}}\right)\right)$-bounded.
\end{lemma}
\begin{proof}
Fix $\omega$ with $\|\omega\| \le 2\sqrt{d}$. 
We have that $Y(\omega, \x_i)$, as a function of $S$, is a random variable that is a sum of a single random variable (the summand that corresponds $\x_i$) that is $\left(\delta, O\left(\sqrt{\log^{c'-1}(1/\delta)}\right)\right)$-bounded, as well as
$(m-1)$ additional i.i.d random variables that have mean $0$, are $\left(\delta, O\left(\sqrt{\frac{\log^{c'}(1/\delta)}{d}}\right)\right)$-bounded, and has second moment $O\left(\frac{1}{d}\right)$. By lemma \ref{lem:hof_like} we have that 
\[
|Y(\omega, \x_i)| \le O\left(\sqrt{\frac{m\log^{c'+1}(1/\delta)}{d}}\right) + O\left(\frac{2m\sqrt{\delta /d}}{1-\delta} \right)
\]
w.p. $1-(m+1)\delta$. Equivalently,
\[
|Y(\omega, \x_i)| \le  O\left(\sqrt{\frac{m\log^{c'+1}(m/\delta)}{d}}\right) + O\left(\frac{2\sqrt{(m+1)\delta /d}}{1-\delta} \right)  =  O\left(\sqrt{\frac{m\log^{c'+1}(m/\delta)}{d}}\right)
\]
w.p. $1-\delta$. We have shown that
\[
\E_{\omega}\E_{S}\left[1\left[|Y(\omega, \x_i)| \ge  O\left(\sqrt{\frac{m\log^{c'+1}(m/\delta)}{d}}\right) 
 \text{ and }\|\omega\|\le 2\sqrt{d}\right]\right] \le \delta
\]
Changing the order of summation and using Markov, we get that w.p. $\ge 1-\sqrt{\delta}$ over the choice of $S$, we have that 
\[
\Pr_{\omega}\left[|Y(\omega, \x_i)| \ge  O\left(\sqrt{\frac{m\log^{c'+1}(m/\delta)}{d}}\right) \text{ and }\|\omega\|\le 2\sqrt{d}\right]\le \sqrt{\delta}
\]
Replacing $\delta$ with $\sqrt{\delta}$ and using the fact that $\log\left(m/\delta^2\right)\le 2\log(m/\delta)$ we get that
that w.p. $\ge 1-\delta$ over the choice of $S$, we have that 
\[
\Pr_{\omega}\left[|Y(\omega, \x_i)| \ge  O\left(\sqrt{\frac{m\log^{c'+1}(m/\delta)}{d}}\right) \text{ and }\|\omega\|\le 2\sqrt{d}\right]\le \delta
\]
Hence, since $\Pr_\omega\left(\|\omega\|>2\sqrt{d}\right) \le 2^{-\Omega(d)}$, we conclude that w.p. $\ge 1-\delta$ over the choice of $S$, 
$Y(\omega,\x_i)$ is $\left(\delta + 2^{-\Omega(d)},O\left(\sqrt{\frac{m\log^{c'+1}(m/\delta)}{d}}\right)\right)$-bounded. Finally, using a union bound, and the fact that $\log\left(m^2/\delta\right)\le 2\log(m/\delta)$ we conclude that w.p. $\ge 1-\delta$ over the choice of $S$, we have that for every $i\in [m]$, $Y(\omega,\x_i)$ is $\left(\delta + 2^{-\Omega(d)},O\left(\sqrt{\frac{m\log^{c'+1}(m/\delta)}{d}}\right)\right)$-bounded.
\end{proof}

\begin{proof} (of lemma \ref{lem:exponential_tail_apx})
By lemma \ref{lem:bouded_w} we conclude that w.p $1-\delta$ over the choice of $S$, for every $i$, $f_\vomega(x_i)$ is an average of $q$ 
i.i.d. $\left(\delta + 2^{-\Omega(d)},O\left(\sqrt{\frac{m\log^{c'+1}(m/\delta)}{d}}\right)\right)$-bounded random variables. Furthermore, the second moment of each of these variables is $O(m)$. Using lemma \ref{lem:hof_like} we have that w.p. $1-(m+1)\delta - m2^{-\Omega(d)}$ over the choice of $\vomega$, 
\[
\left|f_\vomega(\x_i) - f(\x_i)\right| \le O\left(\sqrt{\frac{m\log^{c'+2}(m/\delta)}{dq}}\right)
\]
Using the assumption that $m = d^c$ and simple manipulation we get that w.p. $1-\delta - 2^{-\Omega(d)}$ over the choice of $\vomega$, 
\[
\left|f_\vomega(\x_i) - f(\x_i)\right| \le O\left(\sqrt{\frac{m\log^{c'+2}(m/\delta)}{dq}}\right)
\]
\end{proof}

\subsection{Boundness of distributions}\label{sec:bounded}

Recall that a distribution $\cd$ on $\sphere^{d-1}$ is $R$-bounded if for every $\bu\in\sphere^{d-1}$, $\E_{\x\sim\cd}\inner{\bu,\x}^2 \le \frac{R^2}{d}$. We next describe a few examples of $1$-bounded and $\left(1 + o(1)\right)$-bounded distributions.

\begin{enumerate}
\item The uniform distribution is $1$-bounded. Indeed, for any $\bu\in \sphere^{d-1}$ and uniform $\x$ in $\sphere^{d-1}$ we have
\[
\E_{\x}\inner{\bu,\x}^2 = \sum_{i,j} \E_{\x}u_iu_jx_ix_j = \sum_{i} \E_{\x}u_i^2x_i^2  = \sum_{i} u_i^2 \E_{\x}x_i^2 = \frac{1}{d}\sum_{i} u_i^2 = \frac{\|\bu\|^2}{d} = \frac{1}{d}
\]
\item Similarly, the uniform distribution on the discrete cube $\left\{-\frac{1}{\sqrt{d}},\frac{1}{\sqrt{d}}\right\}^d$ is $1$-bounded. Indeed, for any $\bu\in \sphere^{d-1}$ and uniform $\x$ in  $\left\{-\frac{1}{\sqrt{d}},\frac{1}{\sqrt{d}}\right\}^d$  we have
\[
\E_{\x}\inner{\bu,\x}^2 = \sum_{i,j} \E_{\x}u_iu_jx_ix_j = \sum_{i} \E_{\x}u_i^2x_i^2  = \sum_{i} u_i^2 \E_{\x}x_i^2 = \frac{1}{d}\sum_{i} u_i^2 = \frac{\|\bu\|^2}{d} = \frac{1}{d}
\]
\item Let $\cd$ be the uniform distribution on the points $\x_1,\ldots,\x_m \in \sphere^{d-1}$. Denote by $X$ the $d\times m$ matrix whose $i'$ column is $\frac{\x_i}{\sqrt{m}}$
We have
\begin{eqnarray*}
\max_{\bu\in\sphere^{d-1}}\E_{\x\sim\cd}\inner{\bu,\x}^2 &=& \max_{\bu\in\sphere^{d-1}}\frac{1}{m}\sum_{i=1}^m \inner{\bu,\x_i}^2 
\\
&=& \max_{\bu\in\sphere^{d-1}}\frac{1}{m}\sum_{i=1}^m \bu^T\x_i \x_i^T \bu
\\
&=& \max_{\bu\in\sphere^{d-1}} \bu^T XX^T \bu
\\
&=&  \|X\|^2
\end{eqnarray*}
Hence, $\cd$ is $\|X\|$-bounded. In particular, by standard results in random matrices (e.g. theorem 5.39 in \cite{vershynin2010introduction}), if $\{\x_i\}_{i=1}^m$ are independent and uniform points in the sphere and $m = \omega(d)$ then w.p. $1-o(1)$ over the choice of the points, $\cd$ is $\left(1+o(1)\right)$-bounded.
\item The uniform distribution on any orthonormal basis $\bv_1,\ldots,\bv_d$ is $1$-bounded. Indeed, for any $\bu\in \sphere^{d-1}$ and uniform $i\in [d]$ we have
\[
\E_{i}\inner{\bu,\bv_i}^2 = \frac{1}{d}\sum_{i=1}^d\inner{\bu,\bv_i}^2 = \frac{\|\bu\|^2}{d} = \frac{1}{d}
\]
\end{enumerate}

\bibliographystyle{plainnat}

\bibliography{bib}

\end{document}